\documentclass{article}

\usepackage{microtype}
\usepackage{graphicx}
\usepackage{subfigure}
\usepackage{booktabs} 

\usepackage{hyperref}

\usepackage{amsmath}
\usepackage{amssymb}
\usepackage{cleveref}
\usepackage{amsthm}
\usepackage{wrapfig}
\usepackage{placeins}

\DeclareMathOperator{\Var}{\mathbb{V}ar}

\DeclareMathOperator{\bbeta}{\boldsymbol{\beta}}

\DeclareMathOperator{\btheta}{\boldsymbol{\theta}}

\DeclareMathOperator{\bphi}{\boldsymbol{\phi}}
\DeclareMathOperator{\bvarphi}{\boldsymbol{\varphi}}

\DeclareMathOperator{\bvartheta}{\boldsymbol{\vartheta}}

\DeclareMathOperator{\z}{\mathbf{z}}

\DeclareMathOperator{\x}{\mathbf{x}}

\DeclareMathOperator{\n}{\mathbf{n}}

\DeclareMathOperator{\given}{|}

\DeclareMathOperator{\dL}{\delta_L}
\DeclareMathOperator{\dC}{\delta_C}
\DeclareMathOperator{\dV}{\delta_V}
\DeclareMathOperator{\dbtheta}{d_{\btheta}}

\newcommand{\LogPart}[1]{B\left(#1\btheta_{t-1} + (1-#1)\btheta_0\right)}
\newcommand{\E}[2]{\mathbb{E}_{#1}\!\left[ #2 \right]}
\newcommand{\befsups}[1]{{}^{#1}\!}

\usepackage{mathtools}
\DeclarePairedDelimiterX{\infdivx}[2]{(}{)}{%
  #1\;\delimsize\|\;#2%
}

\usepackage[utf8]{inputenc}
\usepackage[english]{babel}
\newtheorem{proposition}{Proposition}
\newtheorem{assumption}{Assumption}
\newtheorem{lemma}{Lemma}

\newtheorem{corollary}{Corollary}

\Crefname{figure}{Figure}{Figures}
\Crefname{section}{Section}{Section}
\Crefname{algorithm}{Algorithm}{Algorithm}
\crefname{appsec}{Appendix}{Appendices}
\crefname{theorem}{Theorem}{Theorems}
\crefname{lemma}{Lemma}{Lemmas}
\crefname{assumption}{Assumption}{Assumptions}
\crefname{corollary}{Corollary}{Corrolaries}

\usepackage[accepted]{icml2018}

\icmltitlerunning{HAFVF: A HRL algorithm for change detection}

\begin{document}

\twocolumn[
\icmltitle{The Hierarchical Adaptive Forgetting Variational Filter}



\icmlsetsymbol{equal}{*}

\begin{icmlauthorlist}
\icmlauthor{Vincent Moens}{1}
\end{icmlauthorlist}

\icmlaffiliation{1}{COSY, Institute of Neuroscience, Universit\'e Catholique de Louvain, Brussels, Belgium}

\icmlcorrespondingauthor{Vincent Moens}{vincent.moens@uclouvain.be}

\icmlkeywords{Variational Inference, Online learning, Adaptive window, Autoregressive models, Stochastic Gradient Descent}

\vskip 0.3in
]



\printAffiliationsAndNotice{}  

\begin{abstract}
A common problem in Machine Learning and statistics consists in detecting whether the current sample in a stream of data belongs to the same distribution as previous ones, is an isolated outlier or inaugurates a new distribution of data.
We present a hierarchical Bayesian algorithm that aims at learning a time-specific approximate posterior distribution of the parameters describing the distribution of the data observed.
We derive the update equations of the variational parameters of the approximate posterior at each time step for models from the exponential family, and show that these updates find interesting correspondents in Reinforcement Learning (RL). In this perspective, our model can be seen as a hierarchical RL algorithm that learns a posterior distribution according to a certain stability confidence that is, in turn, learned according to its own stability confidence.
Finally, we show some applications of our generic model, first in a RL context, next with an adaptive Bayesian Autoregressive model, and finally in the context of Stochastic Gradient Descent optimization.
\end{abstract}

\section{Introduction}
Learning in a changing environment is a difficult albeit ubiquitous task. One key issue for learning in such context is to discriminate between isolated, unexpected events and a prolonged contingency change. This discrimination is challenging with conventional techniques because they rely on prior assumptions about environment stability. When assuming fluctuating context, past experience will be forgotten immediately when an unexpected event occurs, but if that event was just noise, this erroneous forgetting might be very costly. In less variable contexts, model parameters will tend to change more gradually, thus sometimes missing fluctuations when they happen faster than expected. Most models cover one of the two possibilities, and either gradually adapt their predictions to the new contingency or do it abruptly, but not both.

One classical solution to the problem of change detection is to compare the likelihood of the current observation given the previous posterior distribution with a default probability distribution \cite{Kulhavy1984}, representing an initial, naive state of the learner. Usually, the mixing coefficient (or forgetting factor) that is used to weight these two hypotheses is adapted to the current data in order to detect and account for the possible contingency change. This mixing coefficient can be implemented in a linear or exponential manner \cite{Kulhavy1996}. We will focus here on the exponential case.

In the past decade, several Bayesian solutions to this problem based on the aforementioned strategy have been proposed \cite{Smidl2004,Smidl2012,Azizi2015}. However, they usually suffer from several drawbacks: many of them put a restrictive prior on the mixing coefficient (e.g. \cite{Smidl2004,Masegosa2017}) and cannot account for the fact that an unexpected event is unlikely to be caused by a contingency change if the environment has been stable for a long time. 

We propose the Hierarchical Adaptive Forgetting Variational Filter (HAFVF). The core idea of the model is that the the mixing coefficient can be learned as a latent variable with its own mixing coefficient. It is inspired by the observation that animals tend to decrease their flexibility (i.e. their capacity to adapt to a new contingency) when they are trained in a stable environment and that this flexibility is inversely correlated with the training length \cite{Dickinson1985}. We suggest that this strategy may be beneficial in many environments, where the stability of the system identified by a learner is a variable that can be learned as an independent variable with a certain confidence: in certain environments, contingency changes are inherently more less than in others. Although this assumption may not hold in every case, we show that it helps the algorithm to stabilize and discriminate contingency changes from accidents.

Accordingly, we frame our algorithm in a RL framework. We explore how the forward learning algorithm can be extended to the forward-backward case. We show three applications of our model: first in the case of a simple RL task, next to fit an autoregressive model and finally for gradient learning in a Stochastic Gradient Descent (SGD) algorithm.

\section{Hierarchical Model}

\begin{figure}[t]
\vskip 0.2in
\begin{center}
\centerline{\includegraphics[width=\columnwidth]{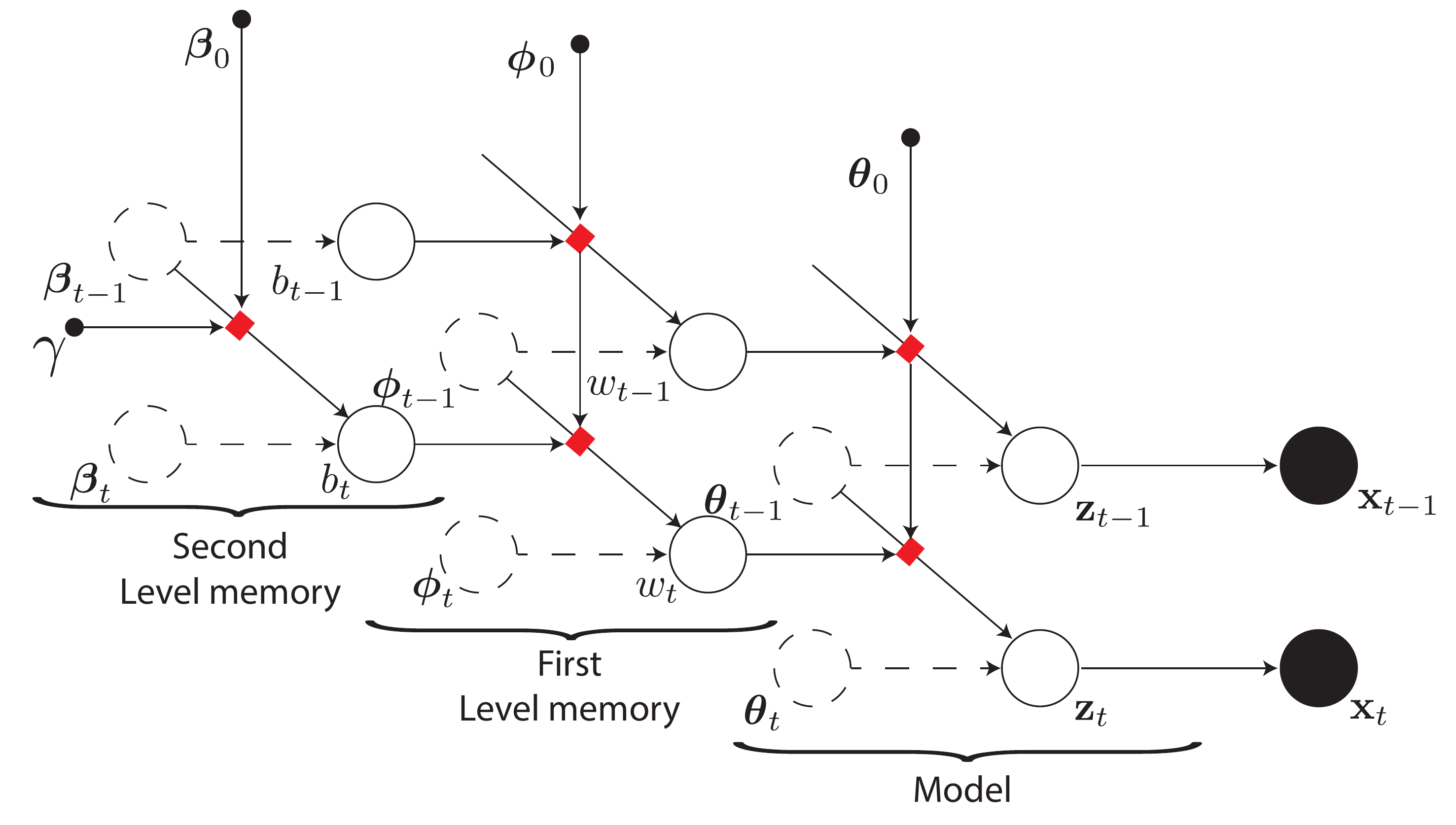}}
\caption{Directed Acyclic Graph of the HAFVF. Latent variables are represented by white circles. Mixture of distributions is represented by red squares. The three levels of the model are displayed: latent variables are distributed with probability $p(\cdot\given\z)$. The prior of latent variable $\z$ is a mixture of the previous posterior distribution and an initial prior with parameters $\btheta_0$. The mixing coefficient is itself distributed according to a similar mixture of distributions with coefficient $b$.}
\label{fig:DAG}
\end{center}
\vskip -0.2in
\end{figure}

Let $\x = \{x_1,x_2,...,x_T\}$ be a stream of data distributed according to a set of $N$ distributions $\textbf{p}=\{p_1(\x_{1:n_1} \given \z_1),... , p_N(\x_{n_{N-1}+1:n_N}\given\z_N)\}$, where the change trials $\n:=\left\{n_1,n_2,...n_N\leq T\right\} \in \mathbb{Z}^+$ are unknown and unpredictable. We make the following assumptions:


\begin{assumption}\label{Assumption1b}
Let $\{t_1, t_2\} \in T$ and $n_x<t_1<t_2$, then $p(t_2 \in \n)\leq p(t_1 \in \n)$.
\end{assumption}
\begin{corollary}\label{corollary1}
If $\text{m}(x_t,p_n)$ is a measure of the relative probability that $x_t$ belongs to $p_n$ wrt $p_0$, and if $x_2=x_1$, then $\text{m}(x_2,p_n)\geq\text{m}(x_1,p_n)$.
\end{corollary}
\Cref{Assumption1b} and \Cref{corollary1} state that the probability of seeing a contingency change decreases with time in a steady environment. This might seem counter-intuitive or even maladaptive in many situations, but it is a key assumption we use to discriminate artifacts from contingency change: after a long sequence, the amount of evidence needed to switch from the current belief to the naive belief is greater than after a short sequence.
This assumption will lead us to build a model where, if the learner is very confident in his belief, it will take him more time to forget past observations, because he will need more evidence for a contingency change. Therefore, in this context, the learner aims not only to learn the distribution of the data at hand, but also a measure of the confidence in the steadiness of the environment.

\begin{assumption}\label{Assumption2}
In the set of probability distributions $\textbf{p}$, all elements have the same parametric form that belongs to the exponential family and have a conjugate prior that is also from the exponential family:
\[
p_n\in\textbf{p} \implies p_n(x_t | \z_n)=h(x_t) \exp\left\{\z^T_n \mathbf{T}(x_t) - A(\z_n)\right\}
\]
and $p_n(\z_n)=\exp\left\{\mathbf{T}(\z_n)^T \btheta - B(\btheta)\right\}.$
\end{assumption}

We now focus on the problem of approximating the current posterior distribution of $\z_{n_x}$ given the current and past observations. For clarity, we will make the $n$ subscripts implicit in the following.
Let us first focus on the problem of estimating the posterior distribution of $\z$ in the stationary case. 
After $t$ steps, and given some prior distribution $p(\z | \btheta_0)$, the posterior distribution can be formulated recursively as:
\begin{equation*}
p(\z | x_t,\x_{<t}) = \left\{ \begin{array}{lll}
         & \frac{p(x_t|\z) p(\z | \x_{<t})}{p(x_t | \x_{<t})} & \text{ for } t>1\\
         & \frac{p(x_t|\z) p(\z | \btheta_0)}{p(x_t)} &\text{ otherwise.}
    \end{array} \right. 
\end{equation*}
Given the restriction imposed by \Cref{Assumption2}, this posterior probability distribution has a closed-form expression and can be estimated efficiently.

We enrich this basic model by first formulating the prior distribution of $\z$ at $t$ as a mixture of the previous posterior distribution and an arbitrary prior:
\begin{equation}\label{eq:MixtDist}
p_t(\z \given \x_{<t} ; \btheta_0, w) =
\frac{p_{t-1}(\z \given \x_{<t})^w p(\z \given \btheta_0)^{1-w}}{Z(w,\x_{<t},\btheta_0)}.
\end{equation}
Following \Cref{Assumption2}, the conjugate distribution $p_{t-1}(\z \given \x_{<t})$ is also from the exponential family and reads
\begin{equation*}
    p_{t-1}(\z| \x_{<t}) := \exp\left\{\z^T \btheta^{\xi} - \btheta^{\eta}A(\z)-B(\btheta)\right\}
\end{equation*}
where we have expanded $\mathbf{T}(\z)$, where $\btheta=\{\btheta^\xi,\btheta^\eta\}$. $\btheta^\eta$ is the part of $\btheta$ that indicates the effective (prior) number of observations. 
If $p_0$ has the same form as $p_{t-1}(\z \given \x_{<t})$, then the log-partition function $Z$ can be computed efficiently \cite{Mandt2014b}: 
\begin{equation*}
\begin{aligned}
Z(w,\btheta_{t-1},\btheta_0) = &\exp \big\{ -wB(\btheta_{t-1}) - (1-w)B(\btheta_0) +\\ &\LogPart{w} \big\}.
\end{aligned}
\end{equation*}
Note that this result simplifies when combined with the numerator of \Cref{eq:MixtDist}:

\begin{equation}\label{eq:MixtExpPrior}
    \begin{aligned}
        p(&\z \given \btheta_{t-1},\btheta_0,w) = \exp \left\{ \mathbf{T}(\z)^T \bvartheta-B(\bvartheta)\right\}
    \end{aligned}
\end{equation}
where $\bvartheta:=w\btheta_{t-1}+(1-w)\btheta_0$. The latent variable $w \in [0;1]$ weights the initial prior with the posterior at the previous trial. We incorporate this variable in the set of the latent variables, and, we put a mixture prior on $w$ with a weight $b$: following this approach, the previous posterior probability of $w$ conditions the current one (similarly to $\z$), together with a prior that is blind to the stream of data up to now. Assuming that $x$, $z$ and $w$ each can be generated by changing distributions, the joint probability now reads:
\begin{align}\label{eq:FullJoint}
\begin{split}
    p(x_t,\z,w,b &| \x_{<t};\btheta_0,\bphi_0,\bbeta_0,\gamma) := p(x_t \given \z) \times\\&\frac{p_{t-1}(\z\given\x_{<t})^w p(\z\given\btheta_0)^{1-w}}{Z(w,\x_{<t},\btheta_0)}\times\\
    &\frac{p_{t-1}(w\given \x_{<t})^b p(w\given\bphi_0)^{1-b}}{Z(b,\x_{<t},\bphi_0)}\times\\
    &\frac{p_{t-1}(b\given \x_{<t})^\gamma p(b\given\bbeta_0)^{1-\gamma}}{Z(\gamma,\x_{<t},\bbeta_0)}
\end{split}    
\end{align}
where we have assumed that the posterior probability $p(\z,w,b|\x)$ factorizes (Mean-Field assumption), and where $\{\btheta_0,\bphi_0,\bbeta_0\}$ are the parameters of the naive, initial prior distributions over $\{\z,w,b\}$ respectively.
The model presented in \Cref{eq:FullJoint} is not conjugate anymore, and the posterior probability does not generally have an analytical solution. We therefore introduce a variational posterior to approximate the posterior probability $p(\z,w,b\given\x)$. In short, Variational Inference \cite{Jaakkola2000} works by replacing the posterior by a proxy of an arbitrary form and finding the configuration of this approximate posterior that minimizes the Kullback-Liebler divergence between this distribution and the true posterior. This is virtually identical to maximizing the Expected Lower-Bound to the log-model evidence (ELBO).

For simplicity, we use a factorized variational posterior $q_t(\z,w,b)=q_t(\z\given\btheta_t)q_t(w\given\bphi_t)q_t(b\given\bbeta_t)$ where each factor has the same form as the prior distribution of its latent variable. Assuming that $q_{t-1}(\cdot)\approx p_{t-1}(\cdot)$ \Cref{eq:FullJoint} conveniently simplifies to:
\begin{equation}\label{eq:FullJointVB}
\begin{aligned}
    p(x_t,\z,w,b &| \x_{<t};\btheta_0,\bphi_0,\bbeta_0,\gamma) \approx p(x_t \given \z) \times\\
    &p(\z\given w (\btheta_{t-1}-\btheta_0)+\btheta_0)\times\\
    &p(w\given b (\bphi_{t-1}-\bphi_0)+\bphi_0)\times\\
    &p(b\given \gamma (\bbeta_{t-1}-\bbeta_0)+\bbeta_0).
\end{aligned}
\end{equation}
This model is shown in \Cref{fig:DAG}.
In what follows, we will restrict our analysis to the case where $w$ and $b$ are Beta distributed, meaning that the approximate posterior we will optimize for these two variables will also be a Beta distribution.

\subsection{Update equations}
\paragraph{Notation} 
We first define the following notation: $\dbtheta:=\btheta_{t-1}-\btheta_0$ is the difference between the previous approximate posterior and the initial prior. We use $\bvartheta:=w\btheta_{t-1}+(1-w)\btheta_0$ as the weighted prior parameters, and $\widehat{\bvartheta}$ as the expectation of $\bvartheta$ under $q(w)$. Similarly, $\bvarphi$ and $\widehat{\bvarphi}$ are the weighted prior over $w$ and its expectation under $q$, respectively. Also, we will often abbreviate the summary statistics of $\z$ as $\mathbf{T}(\z):=\left[\begin{array}{c}
     \z  \\
     -A(\z) 
\end{array}\right]$.

We now focus on the problem of finding the approximate posterior configuration that maximizes the ELBO. Various techniques have been developed to solve this problem: whereas Stochastic Gradient Variational Bayes \cite{Kingma2013} and Stochastic Variational Bayes \cite{Hoffman2013} work well for large datasets, more traditional conjugate \cite{Winn2005} or non-conjugate \cite{Knowles2011} Variational Message Passing (VMP) algorithms are better suited for our problem. This technique indeed allows us to derive closed-form update equations that can be sequentially applied to each of the nodes of the factorized posterior distribution until a certain convergence criterion is met. We interpret these results in a Hierarchical Reinforcement Learning framework, where each level adapts its learning rate (LR) as a function of expected log-likelihood of the current observation given the past.

Fortunately, under the form of the approximate posterior we chose and using Conjugate VMP, the variational parameters of the posterior over the latent parameters $\z$ have a simple form given the current value of $\bphi_t$ and $\bbeta_t$. For a number of $J$ observations observed at time $t$, we have:
\begin{align}\label{eq:updCVMP}
    \btheta^\xi_t &= \widehat{\bvartheta}^\xi + \sum_{j=1}^J \mathbf{T}({x_t}_j)\\
    \btheta^\eta_t &= \widehat{\bvartheta}^\eta + J
\end{align}
\Cref{eq:updCVMP} finds an interesting correspondent in the RL literature. 
Consider the limit case where $\btheta_0=0$ and $J=1$ (which is still analytically tractable following \Cref{eq:MixtExpPrior}).
As the expectation of a distribution of the exponential family has the general form $\E{p(x\given\z)}{\mathbf{T}(x)}=d A(\z)/d \z$, one can derive a similar posterior expectation of $\z$ \cite{Diaconis1979}:
\begin{equation}
    \E{q(\z,w)}{\z}=\frac{\widehat{\bvartheta}^\xi+ \mathbf{T}({x_t})}{\widehat{\bvartheta}^\eta + 1\nonumber}
\end{equation}
Now, replacing $\frac{1}{\bvartheta^\eta + 1}$ by $\alpha$, the above expression becomes \cite{Mathys2016} 
\begin{equation}\label{eq:RL}
    \E{q(\z,w)}{\z}=Q + \alpha (\mathbf{T}({x_t})-Q)\\
\end{equation}
where $Q:=\frac{\btheta_{t-1}^\xi}{\btheta_{t-1}^\eta}$ is the average $\z$ at the time of the previous observation and $\alpha$ is the LR, whose value is inversely proportional to the effective memory $\btheta_{t-1}^\xi$ and to the current expected value of the forgetting factor $\E{q(w)}{w}$\footnote{One can easily see that $\E{q(w)}{w}$ dictates the memory of the learner. If $J=1$ and assuming that $\E{q(w)}{w}$ is stationary, we have: $\lim_{t\to\infty}{\btheta^\eta_t}=\btheta^\eta_0+\frac{1}{1-\E{q(w)}{w}}$}. \Cref{eq:RL} is a classical incremental update rule in RL \cite{Sutton1998}, and our algorithm can be viewed as a special case of such algorithms where the LR is adapted online to the data at hand.

The update equations of $\bphi_t$ is, however, not as simple to derive as $\btheta_t$, because $p(\z \given \btheta_{t-1},\btheta_0,w)$ is not conjugate to its Beta prior $p(w \given \bphi_{t-1}, \bphi_0, b)$. To solve this problem, we used a Non-Conjugate VMP approach \cite{Knowles2011}. Briefly, NCVMP minimizes an approximate KL divergence in order to find the value of the approximate posterior parameters that maximize the ELBO. In order to use NCVMP, the first step is to derive the expected log-joint probability of the model, which we will need to differentiate wrt $\bphi_t$ (or, in the case of the approximate posterior update rule for $b$, $\bbeta_t$). It quickly appears that part of this expression does not always have an analytical form for common exponential distributions: indeed, the expected value of $\E{q(w)}{B(\bvartheta)}$ is, in general, intractable and needs to be approximated. Expanding the Taylor series of this expression around $\widehat{w}=\E{q(w)}{w}$ up to the second order and taking the expectation, we have:
\begin{align}\label{eq:logpart}
    &\E{q(w)}{B(\bvartheta)} \approx B(\widehat{\bvartheta})+ \frac{1}{2}\E{q(w)}{(w-\widehat{w})^2}\triangledown^2_{\widehat{w}}B(\widehat{\bvartheta})
\end{align}
Notice that the second term of the sum in \Cref{eq:logpart} can be expressed as $\frac{1}{2}\Var_{q(w)}[w] \dbtheta^T C(\mathbf{T}(\z) | \widehat{\bvartheta}) \dbtheta$, where $C(\mathbf{T}(\z) | \widehat{\bvartheta})$ is the prior covariance of $\mathbf{T}(\z)$. Hence, this penalty term becomes important when the product of the following factors increase: the distance between the previous posterior and initial prior $\dbtheta$, the posterior variance of $w$ and the prior covariance of $\mathbf{T}(\z)$. This has the effect of favoring values of $\bphi_t$ and $w$ that have a low variance, especially when the two proposed distributions, $q_{t-1}$ and $p_0$, are very distant from each other.

We now derive the update equation for the approximate posterior of $w$. Let us first define
\begin{align}\label{eq:notation}
\begin{split}
        \dL &:= \frac{d}{d \widehat{w}} \E{q(\z)}{\log p(\z \given \widehat{\bvartheta})}\\
    \dC &:= -\frac{1}{2}\Var_{q(w)}[w]\times\\&\frac{d}{d \widehat{w}}\dbtheta^T C(\mathbf{T}(\z)\given\bvartheta) \dbtheta\\
    \dV &:= -\frac{1}{2}\dbtheta^T C(\mathbf{T}(\z)\given\widehat{\bvartheta}) \dbtheta\times\\&C(\log w\given\bphi)^{-1} \nabla_{\bphi}\Var_{q(w)}[w].
\end{split}
\end{align}
We obtain the following result:

\begin{proposition}\label{proposition1}
Using Algorithm 1 of \cite{Knowles2011}, the update equation for $\bphi_t$ has the form:
\begin{align}\label{eq:UpdPhi}
\begin{split}
    \phi^\alpha_t &=\widehat{\varphi}^\alpha+K(\phi^\alpha_t,\phi^\beta_t) \dL + K(\phi^\alpha_t,\phi^\beta_t) \dC + \dV^\alpha \\
    \phi^\beta_t &= \widehat{\varphi}^\beta\underbrace{-K(\phi^\beta_t,\phi^\alpha_t) \dL}_{u_1}\underbrace{- K(\phi^\beta_t,\phi^\alpha_t) \dC}_{u_2} + \underbrace{\dV^\beta}_{u_3}
\end{split}
    \end{align}
\begin{align*}
\begin{split}
        \text{where  } K(x,y) &:=\frac{M x+L(y) y}{(L(x)L(y)-M^2)(x+y)^2}>0\\
        L(x) &:= \psi_1(x)+M\\
        M &:= -\psi_1(\phi^\alpha_t+\phi^\beta_t)\\
\end{split}
\end{align*}
and $\psi_n(\cdot)$ is the n$^{\text{th}}$ order polygamma function.
\end{proposition}
\begin{proof}
Follows directly from Algorithm 1 in \cite{Knowles2011}\footnote{The full development can be found in the supplementary materials.}. 
\end{proof}
The update equation in \Cref{eq:UpdPhi} can be easily transposed for $\bbeta_t$.

In \Cref{proposition1}, we show that the update of $\phi_t$ can be decomposed in four terms: the first is the (weighted) prior $\bvarphi$, which acts as a reference for the update.

The second term, $u_1$, depends upon $\dL$, the derivative wrt $\widehat{w}$ of the expectation of the log probability $p(\z\given\widehat{\bvartheta})$ over $\z$, times a constant $K(\cdot,\cdot)$. $\dL$ has a simple form:
\begin{lemma}\label{lemma1}
The derivative of the first order Taylor expansion of the expected log probability $\log p(\z):=\log p(\z\given\widehat{\bvartheta})$ around $\widehat{w}$ has the form
\begin{equation}\label{eq:dL}
\dL:=\bigg(\left(\E{q(\z)}{\mathbf{T}(\z)}-\E{p(\z)}{\mathbf{T}(\z)}\right)^T\dbtheta\bigg).\nonumber
\end{equation}
\end{lemma}
The proof is given in the supplementary materials. The expression of $\dL$ is easily understood as a measure of similarity between the current update of the variational posterior $q_t(\z)$ and the previous posterior dependent prior $p(\z)$. Note that a rather straightforward result of \Cref{eq:dL} is that $\lim_{\btheta^\eta_{t}\to\infty}\dL=0$: as the posterior becomes stronger, the relative change that one can expect tends to zero, and the impact of $\dL$ on the update of $\bphi$ can be expected to decrease. This is the behaviour we aimed at: a very strong posterior probability becomes more and more difficult to change as the training time increases.

Note also the opposite sign of the $\dL$ related increment in \Cref{eq:UpdPhi} for $\phi^\alpha_t$ and $\phi^\beta_t$. This implies that if $\dL>0$, then $u^\alpha_1>0$, and the update of $\phi^\alpha_t$ will tend to increase. The opposite is true for $\phi^\beta_t$, showing that the posterior of $w$ effectively deals with the similarity between the current observation and the previous ones.

The third and fourth term of \Cref{eq:UpdPhi}, $u_2$ and $u_3$, are conditioned by the posterior variance of $w$ and the prior variance of $\mathbf{T}(\z)$. In brief, they push the update of $\bphi$ in a direction that lowers the variance of both $\btheta_t$ and $\bphi$. We will show in the next section a simple example of the relative contribution that each of these terms has in the update.

An important consideration to make is that the value of $\bphi$ must be $>0$, which implies that $u_1+u_2+u_3>-\bvarphi$, a restriction that may be violated in practice, especially for low values of $\bvarphi$. In such cases, we reset the value of $\bphi_t$ to some arbitrary value (typically $\bphi\gg0$) where the above inequality holds, and resume the update loop until convergence or until a certain amount of iterations is reached. Note that NCVMP is not guaranteed to converge, but, as suggested by \cite{Knowles2011}, the use of a form of exponential damping can improve the convergence of the algorithm.

\subsection{Example: Binary distribution learning}

In order to understand better the relative contribution of $u_1,u_2$ and $u_3$ to the variational update scheme, we generated a sequence of 200 binary data distributed according to a binomial distribution, whose probability was switching between $0.8$ and $0.2$ every 40 trials. This distribution can be modelled as a hierarchy of beta distributions, where the first level is a Bernoulli distribution with a conjugate, Beta approximate posterior, and the one or two levels above are both Beta distributions measuring the stability of the level below. We simulated the learning process in three cases:
\begin{itemize}
\item A two-layer HAFVF model, where only the posterior over $\z$ could be forgotten (incremental). 
\item A two-layer HAFVF model, where the posterior of $w$ was being forgotten at a fixed rate (i.e. $b$ fixed to $0.75$). 
\item A three-layer HAFVF model, where the posterior of $\beta$ was being forgotten at a rate of $\gamma=0.999$.
\end{itemize}

In each of these examples, we used the following implementation: the beta prior of the first level was set to $\btheta_0=1$. The value of $\bphi_0$ was set to $\{0.9,0.1\}$, which showed to be a good compromise between informativeness and freedom to fit the data. If applicable, the top-level was set to $\bbeta=\{0.75,0.25\}$.

In the first case, the fitting rapidly degenerated, as the memory grew at each trial. \Cref{fig:Binary12}, left column, gives a hint about the reason of this behaviour: each observation decreases the prior covariance $C(\mathbf{T}(\z)\given\bvartheta)$, which results in a positive increment for both $\phi^\alpha_t$ and $\phi^\beta_t$ through $u_3$. This can be viewed as a form of confirmation bias: because the posterior over $w$ and $\z$ are confident about the distribution of the data, they tend to reinforce each other and loose flexibility. Consequently, the impact of the contingency changes decreases as learning goes on. This might seem undesirable (and, in this pathological case, it is indeed the case), but in the case of datasets with outliers it can be very beneficial: a longer training in a stable environment will require a longer and/or stronger sequence of outliers to reset the parameters.
\begin{figure}[t]
\vskip 0.2in
\begin{center}
\centerline{\includegraphics[width=0.98\columnwidth]{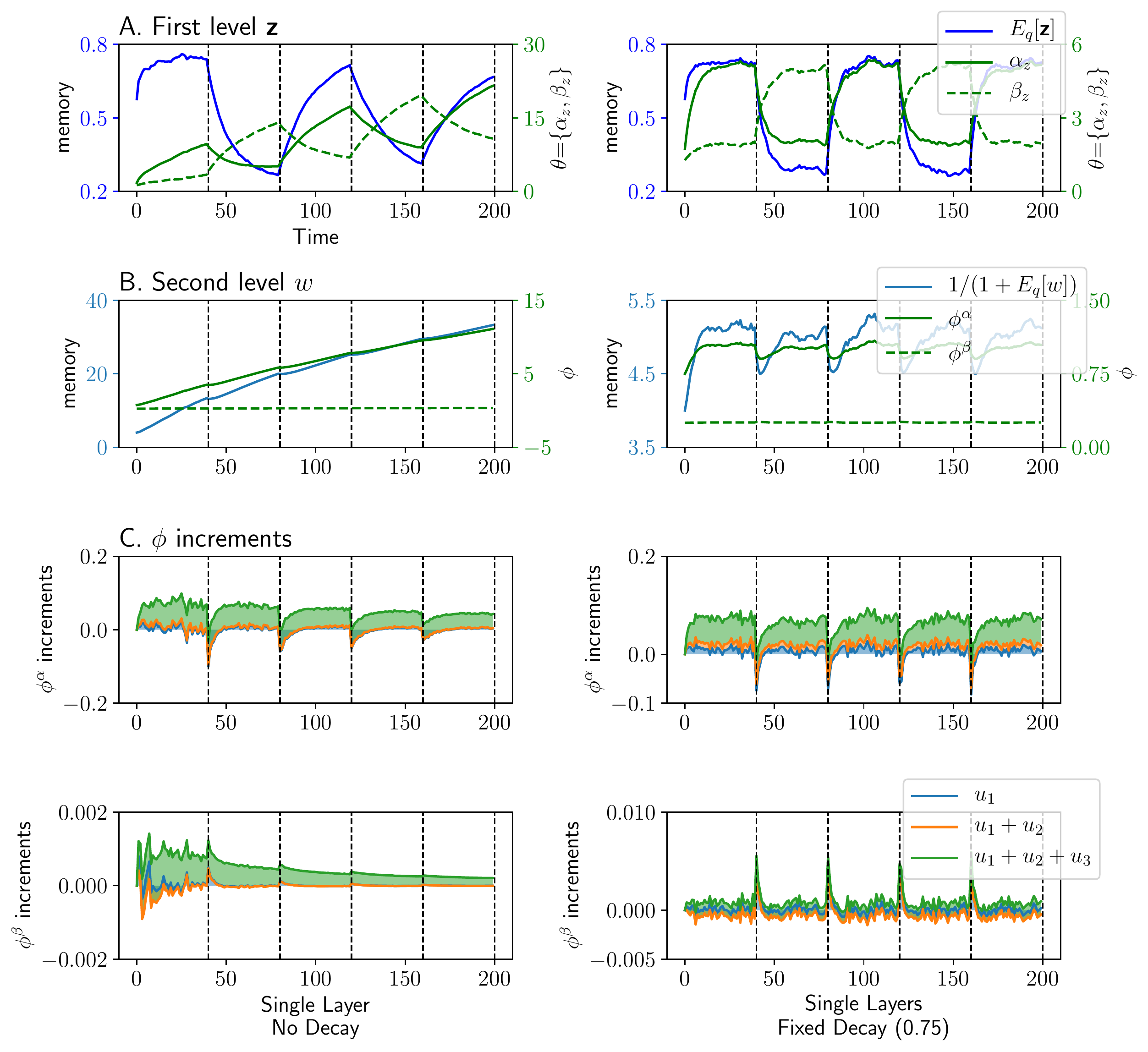}}
\caption{Binary learning with a single level of forgetting. Incremental (\textbf{Left column}) and fixed-decay (\textbf{Right column}) posterior learning of $w$. \textbf{A.} First level learning. The learner looses its capacity to forget as data are observed, because the expected effective memory (\textbf{B.}) tend to grow indefinitely when no decay was assumed. \textbf{C.} Trial-wise increment caused by $u_1,u_2$ and $u_3$. The effects of contingency changes decreased when no decay over $w$ was considered.}
\label{fig:Binary12}
\end{center}
\vskip -0.2in
\end{figure}

Adding a forgetting factor to the posterior of $w$ can moderate the effect of overtraining. In the case of a fixed-forgetting for the posterior probability of $w$ \Cref{fig:Binary12}, right column, the fitting is much more stable: the model is able to learn and forget the current distribution efficiently with a memory bounded at approximately 5 trials (i.e. $\E{q(w)}{w} \approx 0.8$). This shows that adding a forgetting over the posterior of $w$ effectively provides the flexibility we aim at: the contingency changes are efficiently detected, and the drop of $\dL$ (through $u_1$) triggers a resetting of the parameters in the following trials.

In the last case (\Cref{fig:Binary3}), the first level of the model acquires a higher memory than in the second example, due to the ability of the model to adapt the forgetting factor of $w$, which relaxes its bound. It is, however, more flexible than the first example.
\begin{figure}
\begin{center}
\vskip 0.2in
\centerline{\includegraphics[width=0.98\columnwidth]{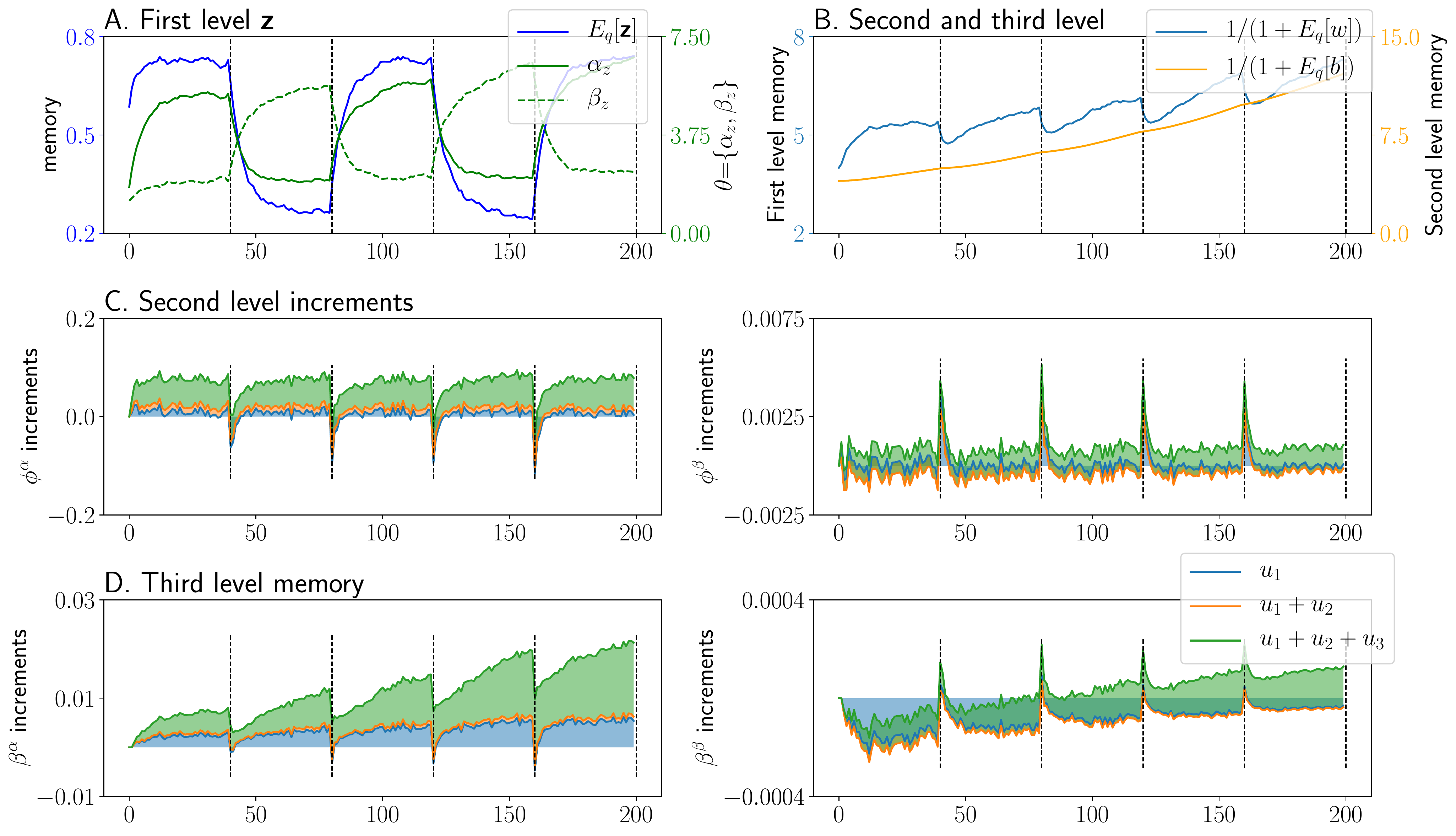}}
\caption{Binary learning with two levels of adaptive forgetting ($w$ and $b$) and a third fixed level $\gamma$. \textbf{D.} is similar to \textbf{C.} for the third level updates $q(b)$.}
\label{fig:Binary3}
\end{center}
\end{figure}



\subsection{Forward-Backward algorithm}
Let us consider the conjugate posterior of the distribution $p(x|\z)$ from the exponential family $p(\z)$ when the whole dataset has been observed. For a given $t \in 1:T$, one can derive the posterior probability of $\z$ given $x_t$ as:
\begin{equation}
    p(\z | x_t, \x_{\neg t};\btheta_0) = \frac{p(x_t\given\z)p(\z\given\x_{<t},\x_{>t};\btheta_0)}{p(x_t\given\x_{<t},\x_{>t})}
\end{equation}
Given \Cref{eq:MixtExpPrior} and \Cref{eq:updCVMP},
if $p(\z\given\btheta_0)$ is the conjugate prior of $p(x_t\given\x)$ and is from the exponential family, we can substitute the prior $p(\z\given\x_{<t},\x_{>t};\btheta_0)$ by $p(\z\given\widetilde{\x}_{<t},\widetilde{\x}_{>t};\btheta_0)$, where $\widetilde{\x}_{<t}$ and $\widetilde{\x}_{>t}$ are the effective samples retrieved from the forward and the backward application of the AFVF on the dataset, respectively. Formally, we have:
\begin{align}\label{eq:updCVMPFB}
    \btheta^\xi_t &= \befsups{f}\btheta^\xi_{t} +\befsups{b}\btheta^\xi_{t} - \sum_{j=1}^J \mathbf{T}({x_t}_j)-\btheta^\xi_0\\
    \btheta^\eta_t &= \befsups{f}\btheta^\eta_{t}+\befsups{b}\btheta^\eta_{t} - J - \btheta^\eta_0\nonumber
\end{align}
where the $f$ and $b$ superscripts index the forward and backward pass, respectively. In offline learning, this technique can increase the effective memory of the approximate posterior distribution just before and after the change trials.

\section{Related work}
Change detection is a broad field in machine learning, where no optimal and general solution exists \cite{KULHAVY1993}. Consequently, assumptions about the structure of the system can lead to very different algorithms and results.

The Kalman Filter \cite{Azizi2015} is a special case of Bayesian Filtering (BF) \cite{Doucet2001} that has had a large success in the Signal Processing literature due to its sparsity and efficiency. It is, however, highly restrictive and its assumptions need to be relaxed in many instances. One can discriminate two main approaches in order to deal with this problem: the first approach is to use a global approximation of BF such as Particle Filtering (PF) \cite{Smidl2008,Smidl2012,Okzan2012}, which enjoys a bounded error but suffers from a lower accuracy than other local approximations. The second class of algorithms comprises the Stabilized Forgetting (SF) family of algorithms \cite{KULHAVY1993,Azizi2015,Laar2017}, from which our model is a special case. SF suffers from an unbounded error, but it usually has a greater accuracy for a given amount of resources \cite{Smidl2008}. Note that SF has been shown to be essential to reduce the divergence between the true posterior and its approximation in recursive Bayesian estimation \cite{Karny2014}. As we apply the SF operator to estimate the posterior of $\z$ and the mixture weight $w$ (through the $b$ weighted mixture prior), we ensure that the divergence is reduced for both of these latent variables.

Even though our model is described as a Stabilized Exponential Forgetting \cite{Kulhavy1996} algorithm and is well suited for signal processing, it can be generalized to models where there is no prediction of future states (e.g. smoothing of a signal, reinforcement learning etc.) Also, it overcomes other methods in several following ways:

First, it uses a Beta prior on the mixing coefficient. This is unusual (but not unique \cite{Dedecius2012}), as most of previous approaches used a truncated exponential prior \cite{Smidl2005,Masegosa2017} or a fixed, linear mixture prior that account for the stability of the process \cite{Smidl2012}. In Stabilized Linear Forgetting, a Bernoulli prior with a Beta hyperprior has been proposed for the mixture weight \cite{Laar2017}. 
Our approach is designed to learn the posterior probability of the forgetting factor in a flexible manner. We show that this posterior probability depends upon its own (and possibly a mixture of) prior distribution and upon the prior covariance of the model parameters $C(\mathbf{T}(\z)\given\widehat{\bvartheta})$. This makes the change detection more subtle than an all-or-none process, as one might observe with a Bernoulli distribution. It also enables us to accumulate evidence for a change of distribution across trials, which can help to discriminate outliers from real, prolonged contingency changes. This is, to our knowledge, an entirely novel feature in the adaptive forgetting literature.

The second important novelty of our model lies in its hierarchical learning of the environment stability. This is somehow similar to the Hierarchical Gaussian Filter (HGF) \cite{Mathys2011,Mathys2014}. The present model is, however, much more general, as the generic form we provide can be applied to several members of the exponential family. Also, although the KL divergence (error term) of our model is not bounded in the long run, it can be efficiently applied to a large subset of datasets and models, whereas the HGF often fails to fit processes that are highly stationary, with many datapoints and/or with abrupt contingency changes.

\section{Experiments}
The HAFVF was coded in the Julia language \cite{Bezanson2017} using a Forward automatic differentiation algorithm \cite{Revels2016} for the NCVMP for the RL and AR parts of this section, and using an analytical gradient for the SGD part.

\subsection{Reinforcement Learning}
We first look at the behaviour of the model in the simple case of estimating the current distribution of a random variable sampled from a moving distribution.
We simulated two sequences of 2x200 datapoints where each pair of points was generated according to the same multivariate normal distribution with mean $\mu_1=\{-2,+2\}$ and $\mu_2=\{+2,-2\}$. We then added an independent random walk to these means.

We applied the Forward-Backward (FB) version of the HAFVF to these datasets. We used the same Normal Inverse Wishart prior for both of these results ($\mu_0=0$, $\kappa_0=0.1$, $\eta_0=3$, $\Lambda_0=I$). The prior over $w$ was manipulated to include a high confidence ($\phi^\alpha_t=9$, $\phi^\beta_t=1$) or a low confidence ($\phi^\alpha_t=0.9$, $\phi^\beta_t=0.1$) about the average value of $w$. Note that both of these priors had the same expected value. To avoid overfitting of early trials (which may happen using weak priors) while keeping the distribution flexible, we used a flat prior over $b$: $\bbeta_0=1$. The top level forgetting was ignored ($\gamma=1$). Results are shown in \Cref{fig:Xp1}.

\begin{figure}[t]
\vskip 0.2in
\begin{center}
\centerline{\includegraphics[width=0.95\columnwidth]{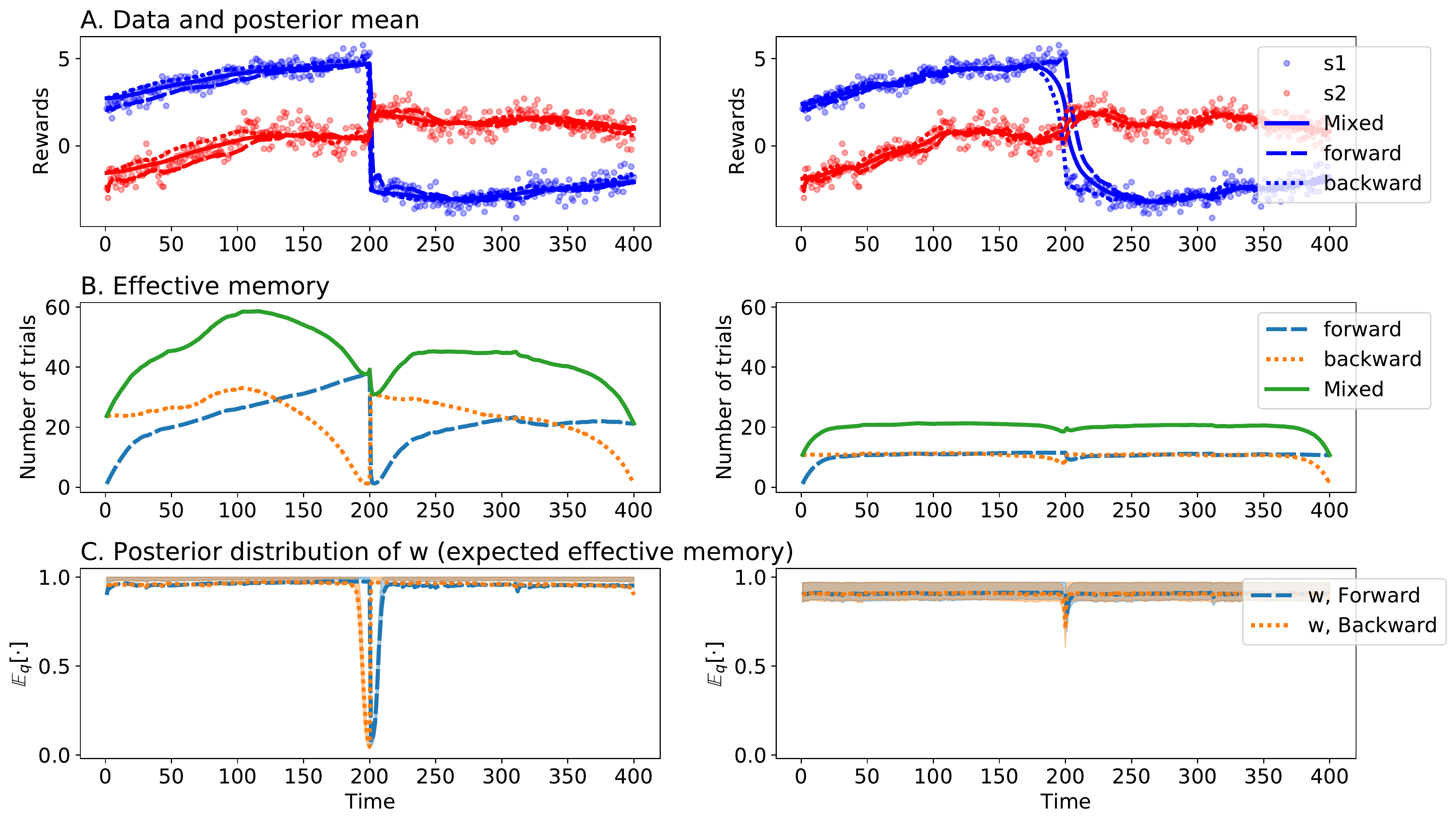}}
\caption{Experiment 1. Left column: weak prior over $w$. Right column: strong prior over $w$. Shaded areas represent the 95\% posterior confidence interval. See text for more comments.}
\label{fig:Xp1}
\end{center}
\vskip -0.2in
\end{figure}

As the first setting had a weak prior over $w$, it had more freedom to adapt the posterior distribution to the current data. The effective memory trace (measured with the parameter $\kappa_t$) was greater when the environment was stable, and changed faster after the contingency change than when the prior was more confident, where the adaptation was slow and the effective memory did not increase much above the prior-defined threshold 10 (or 20 for the FB algorithm). 

The behaviour of both models after the contingency change is informative about the effect that the prior had on the inference process: the weak-prior forgetting factor dropped immediately after an unexpected observation was made, which can be advantageous when sudden changes are expected, but maladaptive in the presence of outliers. The strong-prior model behaved in the opposite way, and handled the change more slowly than its weak-prior counterpart.

It is interesting to note that the posterior probability distribution of $b$ (not shown in the figure) was also more flexible in the first model fit than in the second, because the observations in the level below were also more variable, due to less confident prior over $w$: this had the effect of increasing the gain in precision over $w$, which increased the strength of the posterior over $b$ (through $u_2$ and $u_3$ in \Cref{eq:UpdPhi}).


\subsection{Autoregressive model}
We fitted the HAFVF to a simulated a non-stationary sinusoidal signal of 400 datapoints issued from two separate systems with a low and high frequency. These signals were randomly generated as the sum of five sinusoidal waves, with the scope of observing whether the algorithm was able to adapt to the abrupt contingency change.

Because we also aimed at a more informative view on the performance of the algorithm in the presence of artifacts, we altered this signal by adding two impulses of 2 a.u. at $t=100$ and $t=300$. 

We studied a single implementation of the model, with a relatively strong prior over $w$ ($\bphi=\{4.5,0.5\}$) and a flat prior over $b$, ($\bbeta=\{1.0,1.0\}$). The Forward-backward version of the algorithm was applied. We arbitrarily chose a forward and a backward order of 10 samples. \Cref{fig:Xp2} shows the results of this experiment. 

\begin{figure}[t]
\vskip 0.2in
\begin{center}
\centerline{\includegraphics[width=\columnwidth]{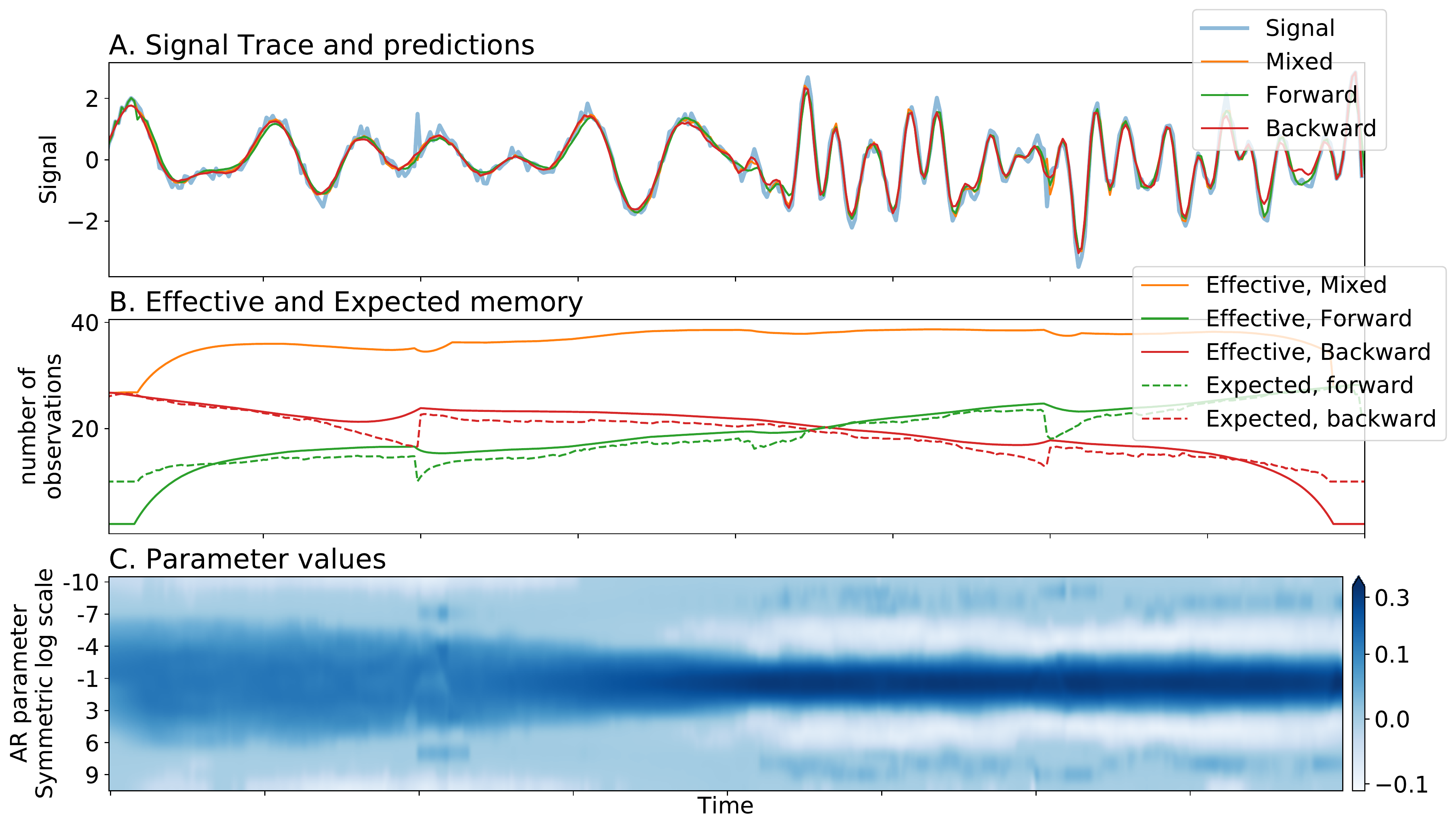}}
\caption{Experiment 2: Autoregressive model with a weak prior over $w$. \textbf{A.} Observations and simulated response of the models. The zoomed windows show the effect of the artifacts on the estimated mean value. \textbf{B.} Effective memory (the "effective number of observations" parameter of the posterior $\btheta_t$: namely $\kappa_t$) of the three parts of the algorithm (plain lines), and corresponding expected effective memory: $1/(1+\E{q(w)}{w}$ (dashed lines). Outliers had a limited impact on learning in both cases. \textbf{C.} Value of the AR mean weights through time. The model dealt adequately with the outliers (as the value of the parameters did not change substentially) and with the contingency change (as the values were adapted to the two different signals).}
\label{fig:Xp2}
\end{center}
\vskip -0.2in
\end{figure}

\subsection{Stochastic Gradient Descent}
SGD is a popular technique to find the minimum of (often computationally expensive) loss-functions over large datasets \cite{Tran2015a} or involving intractable integrals \cite{Kingma2013} that can be sampled from. However, SGD can be unstable, especially with recurrent neural networks \cite{Fabius2014} where an isolated, highly noisy sample in the sequence can lead to a degenerate sample of the gradient over the whole sequence. This effect is further magnified when the sample size is low.

We implemented a slightly modified version of the HAFVF in a SGD framework, intended to be similar to the Adam optimizer\footnote{More details can be found in the supplementary materials} \cite{Kingma2015}. In short, we used two specific decays $w_1$ and $w_2$ for the posterior of the means and variances of the gradients, respectively, while ensuring that $w_1<w_2$. We modelled these posteriors as a set of Normal-Inverse-Gamma distributions. Each set of weights and biases of the multilayered perceptrons was provided with its own hierarchical decay, to take advantage of the fact that some groups of partial derivatives might be more noisy than others. We used this algorithm with a strong prior over $w_1$ and $w_2$ ($\bphi_1\!=\!\{9,1\}$ and $\bphi_2\!=\!\{9.5,0.5\}$), to limit the effect of degenerated gradients on the approximate posteriors.

This algorithm was tested with a variational recurrent auto-encoding regression model inspired by \cite{Moens2017}, where the output probability density was a first passage density of a Wiener process \cite{Ratcliff1978}. The simulated dataset was composed of 64 subjects performing a 500 trials long two alternative forced choice task \cite{Britten1992}, where choices and reaction times were the model was aiming to predict. At each step of the SGD process, 5 subjects were sampled, for a total of 2500 trials. 

\Cref{fig:Xp3} compares the results of the AdaFVF SGD optimizer with the Adam optimizer, executed with the default parameters ($\alpha=0.001$, $\beta_1=0.9$, $\beta_2=0.999$). The AdaFVF showed to be less affected by degenerate samples than Adam, as can be seen from the ELBO trace and from the heat plots of the expected memories, for an estimated average negative ELBO of $1.08$ for Adam and $0.85$ for AdaFVF at the iteration 10000.

\begin{figure}[t]
\vskip 0.2in
\begin{center}
\centerline{\includegraphics[width=0.85\columnwidth]{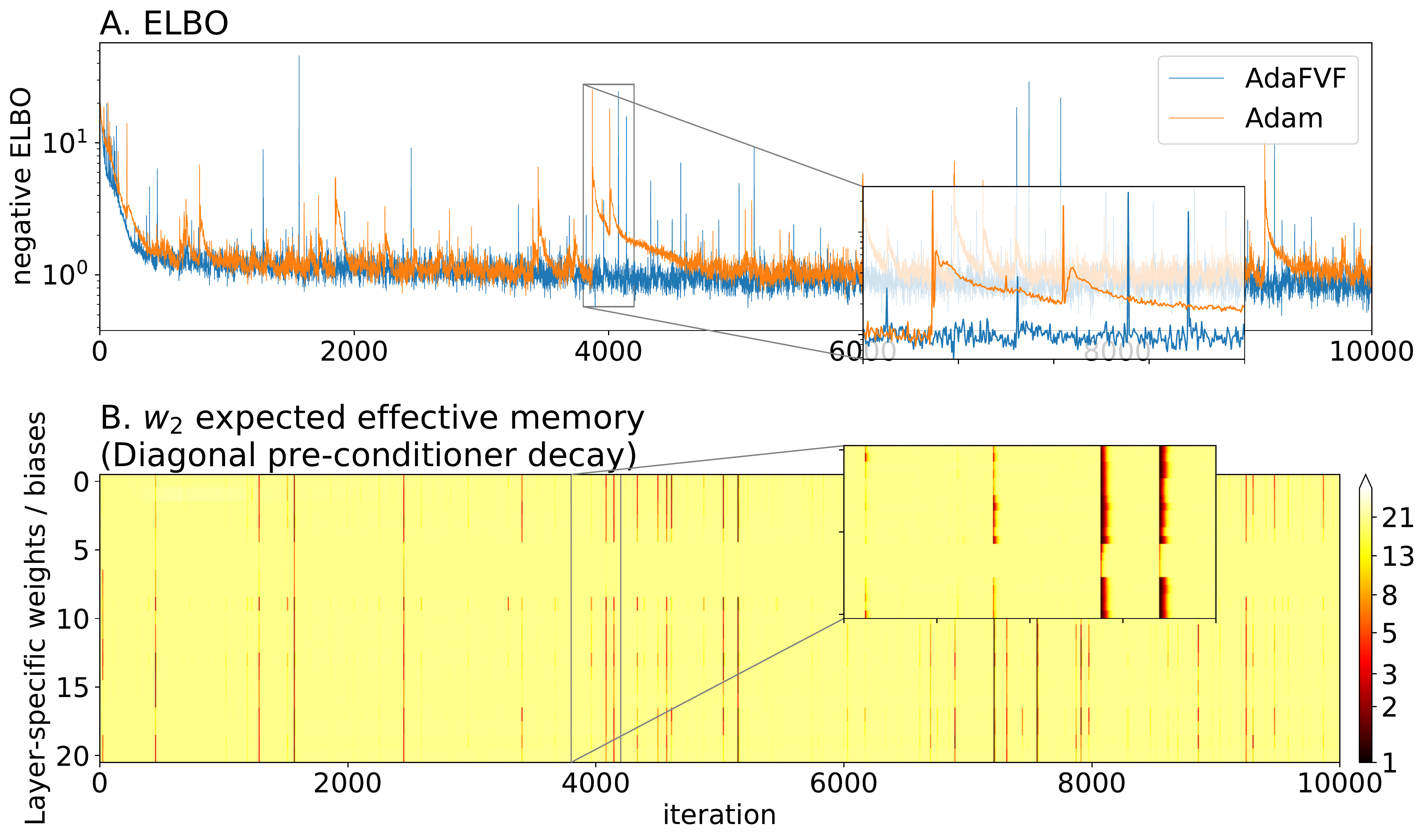}}
\caption{Experiment 3: SGD with the HAFVF. \textbf{A.} ELBO for the Adam optimizer and for the AdaFVF SGD. After an outlier was sampled, the AdaFVF simply forgot the gradient history, and reset its belief to the naive prior, thereby decreasing the relative contribution of this sample. \textbf{B.} Expected memory of the variance posteriors. The impact of outliers is highlighted by the zoomed windows.}
\label{fig:Xp3}
\end{center}
\vskip -0.2in
\end{figure}

\section{Limitations, perspective and conclusion}
Our algorithm has the following limitations: The first  lies in the exponential form we have given to the mixture distributions. A linear form, similar to \cite{Laar2017} could however also be implemented, at specific levels of the hierarchy of the whole model. It may also be difficult to choose an adequate prior on the various levels of the hierarchy. The naive prior of the lower level $p_0(\z)$ is usually crucial but hard to specify, but this is a generic feature in adaptive forgetting. For the two top levels, we propose as a rule of thumb to use a weak prior in situations where abrupt contingency changes are expected. They can also provide a higher memory to the model. They are, however, more affected by outliers than stronger priors. The latter option is therefore advisable in situations where the sequence is expected to contain outliers, and when large amount of data are modelled. There is, however, no generic solution and one might need to try different model specifications before selecting the optimal (i.e. more suited) one. 

The HAFVF and variants could lead to many promising developments in RL related fields, where they might help to prevent unnecessary forgetting of past events during exploration, in signal processing and more distant fields such as deep learning, where they could be used to prevent the occurrence of catastrophic forgetting.

In conclusion, we present a new generic model aimed at coping with abrupt or slow signal changes and presence of artifacts. This model flexibly adapts its memory to the volatility of the environment, and reduces the risk of abruptly forgetting its learned belief when isolated, unexpected events occur. The HAFVF constitutes a promising tool for decay adaptation in RL, system identification and SGD.

\pagebreak

\section*{Acknowledgements}


We thank the reviewers for their careful reading and precious comments on the manuscript. We also thank Oleg Solopchuk and Alexandre Z\'enon, who greatly contributed to the development and writing of the present manuscript. The present work was supported by grants from the ARC (Actions de Recherche Concert\'ees, Communaut\'e Francaise de Belgique).


\bibliography{library}
\bibliographystyle{icml2018}

\end{document}